\documentclass{article}
\usepackage{ijcai22}

\usepackage{times}
\usepackage{soul}
\usepackage{url}
\usepackage[hidelinks]{hyperref}
\usepackage[utf8]{inputenc}
\usepackage[small]{caption}
\usepackage{graphicx}
\usepackage{amsmath}
\usepackage{amsthm}
\usepackage{booktabs}
\usepackage{algorithm}
\usepackage{algorithmic}
\usepackage{amssymb}
\usepackage{subcaption}
\usepackage{cleveref}
\usepackage{nidanfloat}
\usepackage{color}

\newtheorem{theorem}{Theorem}

\author{
Nikita Muravev$^{1,2}$
\And
Aleksandr Petiushko$^{2,3}$\footnote{Work done at Huawei Moscow Research Center, Moscow, Russia}\\
\affiliations
$^1$Huawei Moscow Research Center\\
$^2$Lomonosov Moscow State University\\
$^3$Nuro, Inc.\\
\emails
nikita.muravev@huawei.com,
apetiushko@gmail.com
}

\title{Certified Robustness via Randomized
Smoothing over\\ Multiplicative Parameters of Input Transformations}

\newtheorem{lemma}[theorem]{Lemma}
\newtheorem{definition}[theorem]{Definition} 

\begin{document}

\maketitle

\begin{abstract}
    Currently the most popular method of providing robustness certificates is randomized smoothing where an input is smoothed via some probability distribution.
	We propose a novel approach to randomized smoothing over multiplicative parameters. Using this method we construct certifiably robust classifiers with respect to a gamma correction perturbation and compare the result with classifiers obtained via other smoothing distributions (Gaussian, Laplace, uniform). The experiments show that asymmetrical Rayleigh distribution allows to obtain better certificates for some values of perturbation parameters. To the best of our knowledge it is the first work concerning certified robustness against the multiplicative gamma correction transformation and the first to study effects of asymmetrical distributions in randomized smoothing.
\end{abstract}

\section{Introduction}
It is well known \cite{szegedy2014intriguing} that modern classification models are vulnerable to adversarial attacks. There were many attempts \cite{43405,35184,Liu_2018_ECCV} to construct empirically robust classifiers, but all of them are proven to be vulnerable to some new more powerful attacks \cite{10.1145/3128572.3140444,pmlr-v80-athalye18a}
. As a result, a lot of certification methods have been introduced \cite{wong,li2020provable,alfarra2021deformrs}
that provide provably correct robustness certificates. Among this variety of certification methods the randomized smoothing remains one of the most effective and feasible approaches which is easy to implement and to parallel. Randomized smoothing \cite{cohen,li,lecuyer} can be applied to data and classifiers of any nature since it does not use any information about the internal structure of the classifier. It is also easily scalable to large networks unlike other methods. All these benefits lead to the fact that the randomized smoothing approach is currently one of the most popular and powerful solutions to robustness certification tasks.

Consequently, there are a lot of works devoted to the randomized smoothing. While most of them concern a setting of robustness within $l_p$- or $l_{\infty}$-balls, some generalize the approach to cover parameterized transformations \cite{fischer,li2020provable}. These parameters usually add up to each other when composing transformations, which allows to smooth classifiers over them. Yet there are some other types of perturbations such that their parameters do not add up but multiply to each other (e.g. a volume change of audio signals or a gamma correction of images). Though it is possible to reduce this case to the additive one by a logarithm (like it is done in \cite{fischer}), we find that alternative distributions (like Rayleigh) may result in better certificates. Thus we propose a new approach of smoothing directly over multiplicative parameters. We prove the theoretical guarantees for classifiers smoothed with Rayleigh distribution and compare them with those obtained via other smoothing distributions for the gamma correction perturbation. Our experiments show that the Rayleigh smoothing cannot be consistently outperformed by the others and thus may be more suitable for certain values of the perturbation parameter.

The main contributions of this work are:
\begin{itemize}
	\item Direct generalization of the randomized smoothing method to the case of multiplicative parameters;
	\item Practical comparison of the proposed method and other smoothing techniques for the gamma correction perturbation, which shows superiority of our method for some factors;
	\item Construction of the first certifiably robust classifiers for the gamma correction perturbation.
\end{itemize}

\section{Background}

The randomized smoothing is a method of replacement of an original classifier $f$ with its ``smoothed'' version ${g(x) = \arg\max\limits_{c}{\mathbb{P}_{\beta}(f\circ\psi_{\beta}(x)=c)}}$ that returns the class the base classifier $f$ is most likely to return when an instance is perturbed by some randomly distributed composable transformation $\psi$. A right choice of the distribution $\beta$ allows to predict a robustness certificate for a given input $x$ for this smoothed classifier $g$ using its confidence in its own answer. Here a robustness certificate is a range of transformations such that they do not change the predicted class $g(x)$ when applied to the given input $x$. One expects the new (smoothed) classifier to be reasonably good if the base one is sufficiently robust.

This technique has only two major drawbacks. The first one is that the transformation we want out classifier to be robust against has to be composable. It is a strong constraint since even theoretically composable transformations lose this property in practice due to rounding and interpolation errors. The other one is that generally we are unable to calculate the smoothed classifier $g$ directly and instead have to use approximations. Since we cannot calculate the actual prediction of $g$, we use the Monte-Carlo approximation algorithm with $n$ samples to obtain the result with an arbitrary high level of confidence. The introduced parameter $n$ controls the trade-off between the approximation accuracy and the evaluation time on inference. Though it is possible to smooth a default undefended base classifier $f$, experiments show that the best results can be achieved by smoothing of already (empirically) robust base classifier $f$. Thus during training of the base classifier $f$ we use the same type of augmentation we expect the smoothed classifier $g$ to be robust against. For a more detailed description of the randomized smoothing method see \cite{cohen,li,lecuyer}.

\section{Generalization of smoothing for multiplicative parameters}

Let us extend the notion of composable transformations in case of multiplicative parameters.

\begin{definition}
	A parameterized map ${\psi_{\delta} : X \rightarrow X,} \; {\delta\in \mathcal{B}\subset\mathbb{R}^n}$ is called multiplicatively composable if 
	\begin{equation}(\psi_{\delta}\circ\psi_{\theta}) (x) = \psi_{(\delta\cdot \theta)} (x), \; \forall x\in X,\; \forall \delta,\theta\in \mathcal{B},
	\end{equation}
	where $\delta\cdot \theta\in\mathcal{B}$ means the element-wise multiplication of vectors.
\end{definition}

Usually multiplicative parameters are positive, thus one needs probability distributions with positive supports for randomized smoothing. Yet unlike the additive case an identity transformation corresponds to the parameter value equal to 1 rather than 0. So a conventional exponential distribution \cite{li2020provable} is unsuitable here (we want to concentrate probability mass at 1). Thus we propose a Rayleigh distribution for these needs.

\begin{definition}
	A random variable $\zeta$ has a Rayleigh distribution with the scale parameter $\sigma > 0$ $(\zeta\sim Rayleigh(\sigma))$ if its probability density function (PDF) has a form \begin{equation}
	p_{\zeta}(z) = \sigma^{-2}z e^{-z^2/(2\sigma^2)}, z \geq 0.
	\end{equation}
\end{definition}

For classifiers smoothed with the Rayleigh distribution the following robustness guarantee can be obtained:

\begin{theorem}\label{main}
	Let ${x\in {\mathbb{R}^m}}$, ${f:\mathbb{R}^m \rightarrow Y}$ be a classifier, ${\psi_{\beta}:\mathbb{R}^m \rightarrow \mathbb{R}^m}$ be a multiplicatively composable  transformation for ${\beta\sim Rayleigh(\sigma)}$ and ${g(x) = \arg\max\limits_{c}{\mathbb{P}_{\beta}(f\circ\psi_{\beta}(x)=c)}}$. Denote
	$${p_A = \mathbb{P}_{\beta}(f\circ \psi_{\beta}(x)=c_A)},\;\;
	{p_B = \max\limits_{c_B \not= c_A}{ \mathbb{P}_{\beta}(f\circ\psi_{\beta}(x)}=c_B)}.
	$$
	If
	$$p_A \geq \underline{p_A} > \overline{p_B}\geq p_B,$$
	then $g\circ\psi_{\gamma}(x)=c_A$ for all $\gamma$ satisfying $\gamma_1 < \gamma < \gamma_2$, where $\gamma_1, \gamma_2$ are the only solutions of the following equations
	\begin{equation}
	F(\gamma_1^{-1}F^{-1}(\overline{p_B})) + F(\gamma_1^{-1}F^{-1}(1-\underline{p_A}))=1,
	\end{equation}
	\begin{equation}
	F(\gamma_2^{-1}F^{-1}(\underline{p_A})) + F(\gamma_2^{-1}F^{-1}(1-\overline{p_B}))=1,
	\end{equation}
	and $F(z) = 1 - e^{-z^2/(2\sigma^2)}$ is the CDF of $\beta$.\end{theorem}

A proof is similar to the one provided in \cite{fischer}, though it is also possible to obtain the same result using the constrained adversarial certification framework \cite{zhang} or TSS framework \cite{li2020provable}. All variants are given in Appendix \ref{A},\ref{B}. Notice that the certificate bounds $\gamma_1, \gamma_2$ can be easily found numerically (e.g. by a simple bisection method). But it is also possible to use a trivial estimation $\overline{p_B}=1-\underline{p_A}$ for the probability of the top two class. In that case the equations can be solved analytically:
\begin{equation}
\gamma_1 = \frac{F^{-1}(1-\underline{p_A})}{F^{-1}(\frac{1}{2})}=\sqrt{\frac{\log\underline{p_A}}{\log\frac{1}{2}}},
\end{equation}
\begin{equation}
\gamma_2 = \frac{F^{-1}(\underline{p_A})}{F^{-1}(\frac{1}{2})}=\sqrt{\frac{\log(1-\underline{p_A})}{\log\frac{1}{2}}}.
\end{equation}

The only parameter to be chosen for the Rayleigh distribution is the scale $\sigma$. It seems reasonable to choose it in such a way that either the median or the mean of the random value equals one (value multiplied by 1 stays the same). We try both variants and find that the median equal to 1 results in better certificates (See Appendix \ref{C} for comparison details). Thus hereafter the scale is equal to $\frac{1}{\sqrt{2\ln{(2)}}}$.

The above theorem can be generalized in case of transformations with multiple multiplicative parameters.

\begin{theorem}
	Let $x\in {\mathbb{R}^m}$, $f:\mathbb{R}^m \rightarrow Y$ be a classifier, ${\psi_{\beta}:\mathbb{R}^m \rightarrow \mathbb{R}^m}$, ${\beta=(\beta_1, ..., \beta_n)^{T}}$ be a multiplicatively composable  transformation for independent and identically distributed random variables $\beta_i\sim Rayleigh(\sigma)$ and ${g(x) = \arg\max\limits_{c}{\mathbb{P}_{\beta}(f\circ\psi_{\beta}(x)=c)}}$.
	Denote $${p_A = \mathbb{P}_{\beta}(f\circ \psi_{\beta}(x)=c_A)},\;\; {p_B = \max\limits_{c_B \not= c_A}{ \mathbb{P}_{\beta}(f\circ\psi_{\beta}(x)}=c_B)}.$$ If
	$$p_A \geq \underline{p_A} > \overline{p_B}\geq p_B,$$
	then $g\circ\psi_{\gamma}(x)=c_A$ for all $\gamma\in\Omega$, where $\Omega$ is a region defined by the following inequality
	
	\begin{multline}
	\mathbb{P}((\gamma_1^{2}-1)\beta_1^2 +...+ (\gamma_n^{2}-1)\beta_n^2 \leq r) \\>\mathbb{P}((\gamma_1^{2}-1)\beta_1^2 +...+ (\gamma_n^{2}-1)\beta_n^2 \geq \theta),
	\end{multline}
	where $r$ and $\theta$ are the only solutions of the following equations
	\begin{equation}
	\mathbb{P}((1-\gamma_1^{-2})\beta_1^2 +...+ (1-\gamma_n^{-2})\beta_n^2 \leq r)=\underline{p_A},
	\end{equation}
	\begin{equation}
	\mathbb{P}((1-\gamma_1^{-2})\beta_1^2 +...+ (1-\gamma_n^{-2})\beta_n^2 \geq \theta)=\overline{p_B}.
	\end{equation}
\end{theorem}

A proof is analogous to the one for a single parameter case and can be found in Appendix \ref{A}.

Usually we cannot evaluate the probabilities $p_A, p_B$. Thus we use the Clopper-Pearson \cite{clopper} bounds $p_B \leq \overline{p_B} < \underline{p_A} \leq p_A$ that can be calculated with an arbitrary high confidence probability $1-\alpha$. In case $\underline{p_A} \leq \frac{1}{2}$ the classifier $g$ abstains from answering (i.e. returns the ``abstain'' answer).

We present $\gamma_1, \gamma_2$ values for some $\underline{p_A}, \overline{p_B}$ in Table \ref{gammas} to show the certificates that can possibly be achieved.
\begin{table}[t]
\centering
\begin{tabular}{@{}rrrr@{}}
    \toprule
	$\underline{p_A}$ & $\overline{p_B}$ & $\gamma_1$ & $\gamma_2$\\ \midrule
        0.600 & 0.400 & 0.86 & 1.15 \\
		 & 0.200 & 0.71 & 1.33 \\
		0.700 & 0.300 & 0.72 & 1.32 \\
		 & 0.100 & 0.54 &  1.56 \\
		0.800 & 0.200 & 0.57 & 1.52 \\
		0.900 & 0.100 & 0.39 & 1.82 \\
		0.990 & 0.010 & 0.12 & 2.58 \\
		0.999 & 0.001 & 0.04 & 3.16 \\ \bottomrule
\end{tabular}
\caption{The calculated robustness certificates $(\gamma_1, \gamma_2)$ for the top two class probabilities $\underline{p_A}, \overline{p_B}$.}
\label{gammas}
\end{table}

\begin{figure*}[b]

\subcaptionbox{$\gamma = 0.5$}{\includegraphics[width=0.33\textwidth]{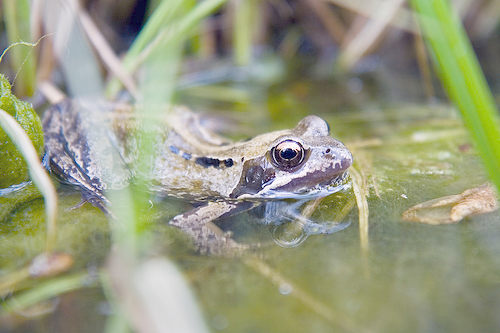}}
\subcaptionbox{$\gamma = 1$}{\includegraphics[width=0.33\textwidth]{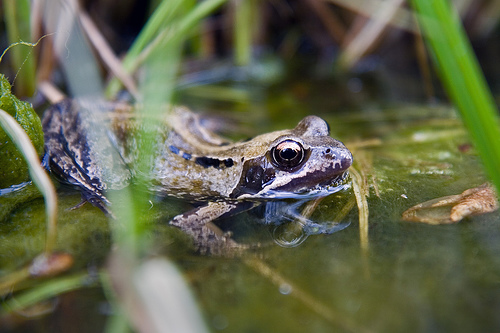}}
\subcaptionbox{$\gamma = 2$}{\includegraphics[width=0.33\textwidth]{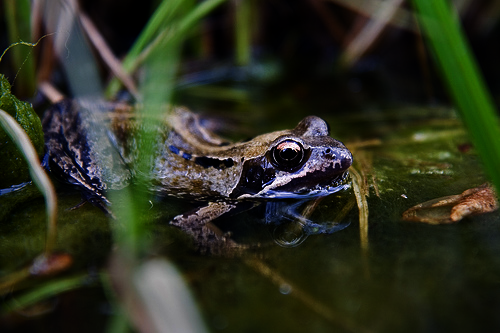}}

\caption{Demonstration of the gamma correction transformation on a frog image.}
\label{frog}
\end{figure*}

\section{Experiments}

There are a lot of multiplicatively composable transformations. We choose the gamma correction of images for our experiments due to its importance and simplicity.

Gamma correction is a very popular operation in image processing. There are a lot of works on how it is used for image enhancement \cite{VELUCHAMY2019329} and medical image processing \cite{Karuppanagounder,AGARWAL2017509}.
Yet to the best of our knowledge no works on certified robustness to gamma correction attacks have been released. Thus we are eager to both fill this gap and try out our new smoothing technique on it.

If we consider an input image in the RGB-format as a tensor $x$ with entries in $[0,1]$, then the gamma correction $G_{\gamma}$ with the gamma factor $\gamma$ simply raises $x$ to the power $\gamma$ in the element-wise manner: $G_{\gamma}(x)=x^{\gamma}$.
For a human eye it looks like change of brightness (Fig.\ref{frog}).

Obviously, this transformation is multiplicatively composable:
\begin{equation}
G_{\beta}\circ G_{\gamma}(x)=(x^{\gamma})^{\beta}=x^{\gamma\cdot\beta}=G_{\gamma\cdot\beta}(x).
\end{equation}
But it should be noticed that in reality a gamma corrected image is likely to be converted to some image format with colour channel of limited width afterwards. In that case some information is lost and the resulting transformation is no longer composable. Thus we have two settings: 1) ``idealized'' --- when the colour channel is so wide that the conversion error is negligible in comparison with the rounding error of a machine; 2) ``realistic'' --- when the conversion error is significant and cannot be ignored.

Other works usually consider idealized setting where the only type of error we encounter is an interpolation error. Following their lead we conduct most of our experiments in this setting though we also show how to deal with conversion errors.

\subsection{Idealized setting}

As it is mentioned before the multiplicative parameters $\beta, \gamma$ can be converted into the additive ones $a = \log_c{\gamma}$, $b =\log_c{\beta}$ via logarithm: \begin{equation}
\gamma\cdot\beta = c^{\log_c{\gamma}}\cdot c^{\log_c{\beta}} = c^a \cdot c^b = c^{a+b},
\end{equation}
where $c$ is some fixed base. Thus one can smooth over these new parameters with the standard Gaussian distribution to obtain certificates on the original ones. The certified accuracy for a factor $\gamma$ is the proportion of correctly classified test images whose certificates contain the value $\gamma$. All certificates are obtained with the mistake probability $\alpha=0.001$.

\begin{figure*}[t]
\subcaptionbox{ResNet-20 on CIFAR-10}{\includegraphics[width=0.45\textwidth]{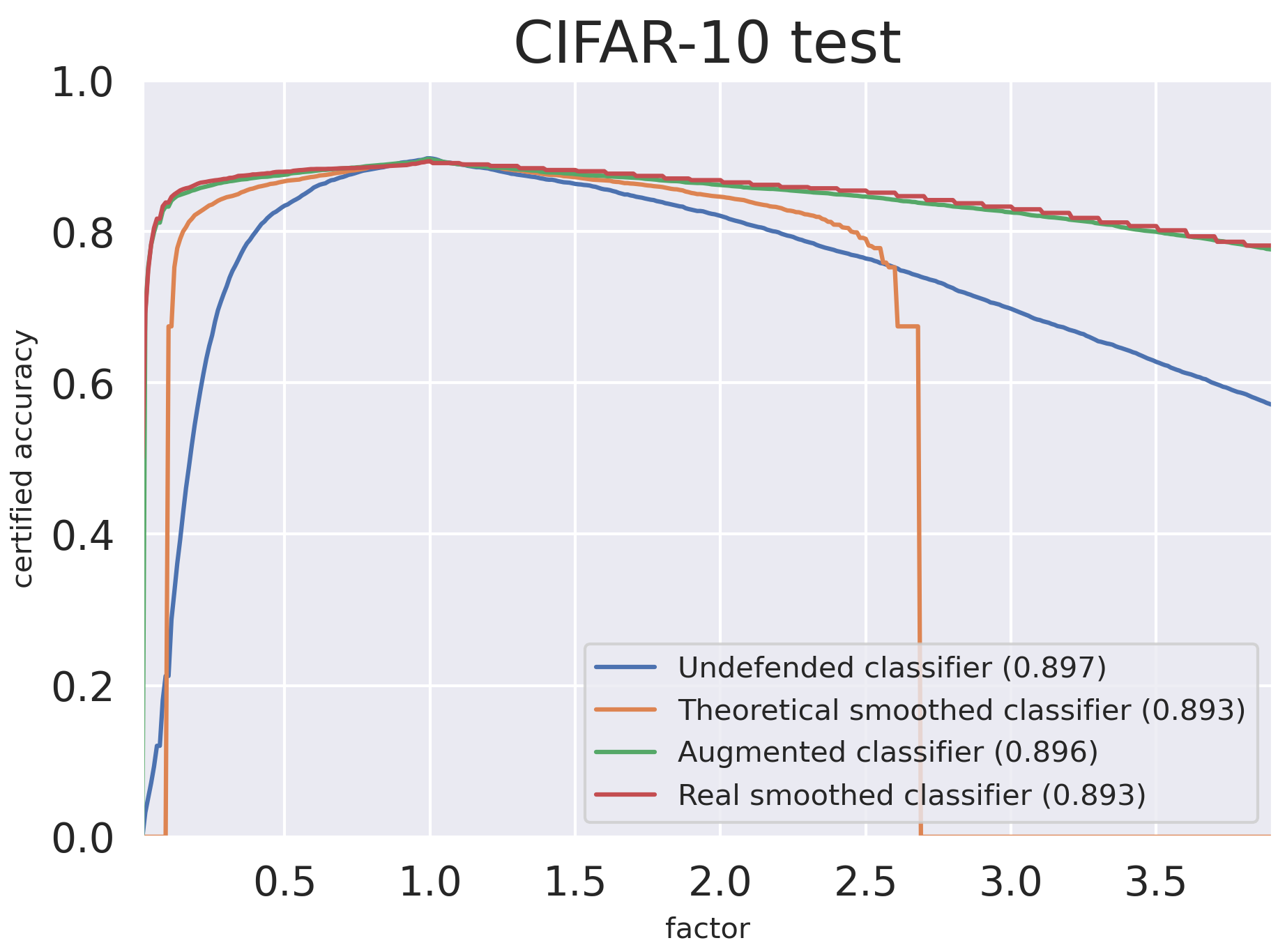}}
\hspace{1.2 cm}
\subcaptionbox{ResNet-50 on ImageNet}{\includegraphics[width=0.45\textwidth]{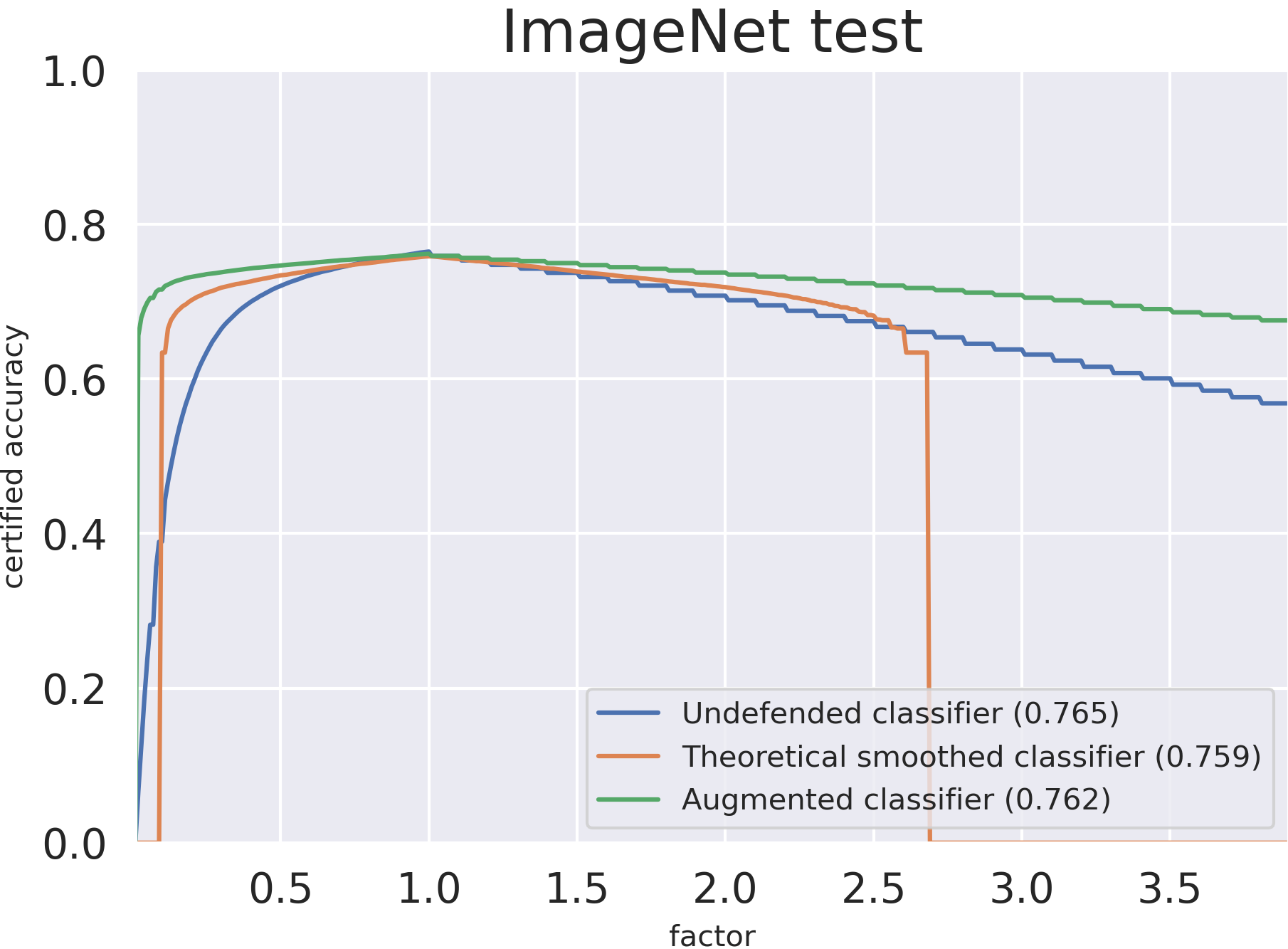}}

\caption{Comparison of theoretically predicted certificates with empirically obtained ones for the gamma correction perturbation.}
\label{comparison}
\end{figure*}

We compare the theoretical certificates for the smoothed classifier with empirical ones for: a) the same smoothed classifier (real smoothed classifier), b) the base architecture without augmentations on training or smoothing (undefended classifier), and c) the base architecture with augmentations on training but without smoothing (augmented classifier) (Fig.\ref{comparison}a). Numbers in brackets show values at 1, that is benign or clean accuracy: the portion of unperturbed images that were correctly classified.

The empirical certificates are obtained via evaluations of classifiers' predictions on perturbed images. For every image $x$ starting from gamma factor $\gamma_0=1$ we evaluate classifier's predictions on $G_{\gamma_k}(x)$ for every $\gamma_k=\gamma_{k-1}+\epsilon$ with a step $\epsilon=0.01$. If the classifier predicts a wrong label on the input $G_{\gamma_k}(x)$ for some $\gamma_k$, we assume that it is robust only in the interval $[1, \gamma_{k-1}]$ and stop the certification procedure. The same procedure is used to find the left end of a certification interval. We use $\epsilon=0.1$ searching for the right ends of certificate intervals for some classifiers (you can spot them by the less smooth right parts of corresponding graphs) due to computational complexity of this process. The resulting certificate can be over optimistic but still it provides a good approximation of the actual robustness. One can see that the real accuracy is bounded from below by the theoretically guaranteed one, which is expected.

We also observe that the smoothed classifier has almost the same actual robustness as the just augmented one. Thus it can be concluded that smoothing provides theoretical certificates at a cost of extra time on inference but does not significantly impact the actual robustness of the augmented base classifier.

The randomized smoothing method is easily scalable to deep models and large datasets. We repeat certification experiments for ResNet-50 on ImageNet \cite{inproceedings} (Fig.\ref{comparison}b) but do not approximate actual robustness due to computational complexity of this process.

\begin{figure*}[b!]
\includegraphics[width=0.33\textwidth]{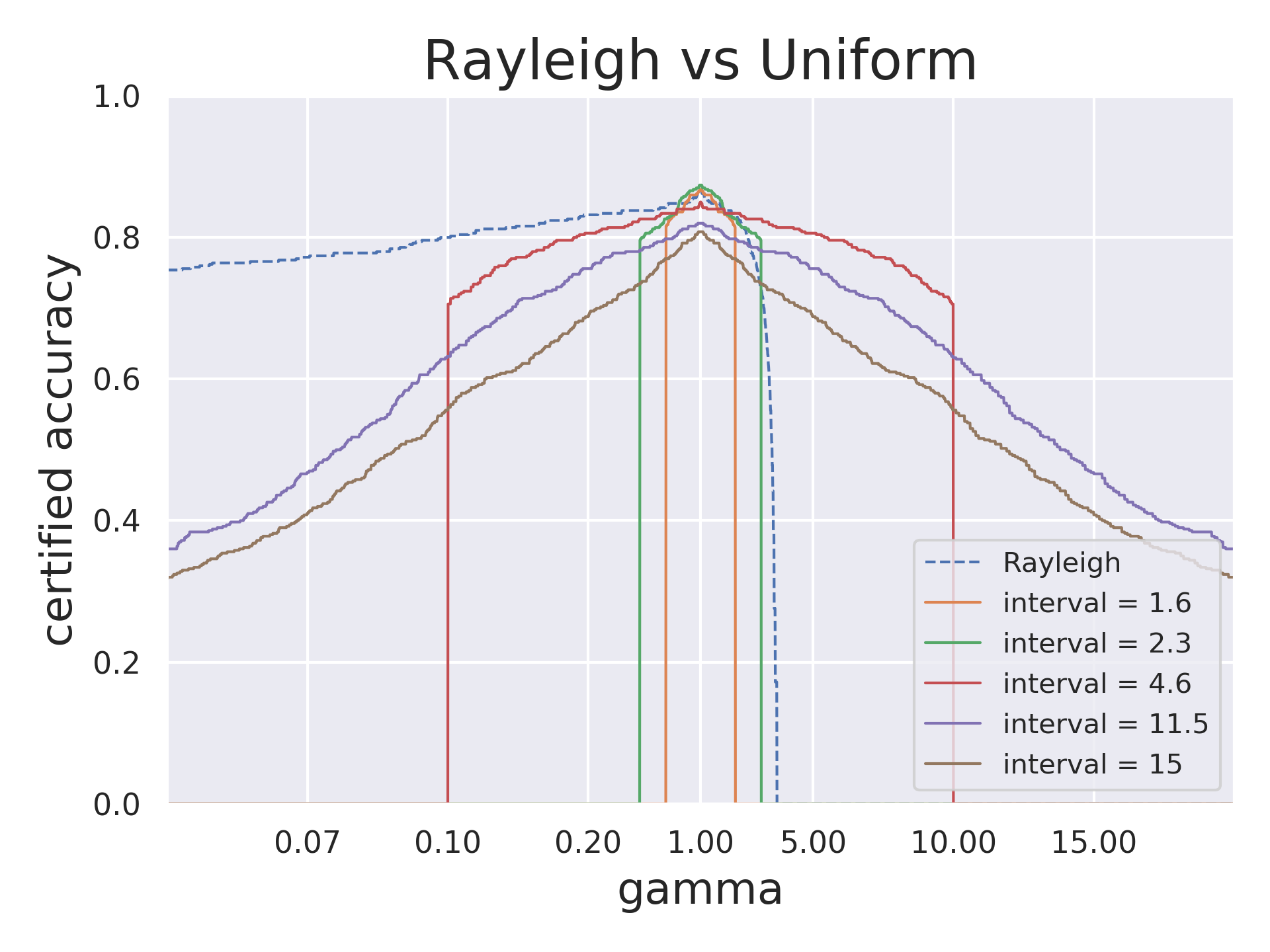}
\includegraphics[width=0.33\textwidth]{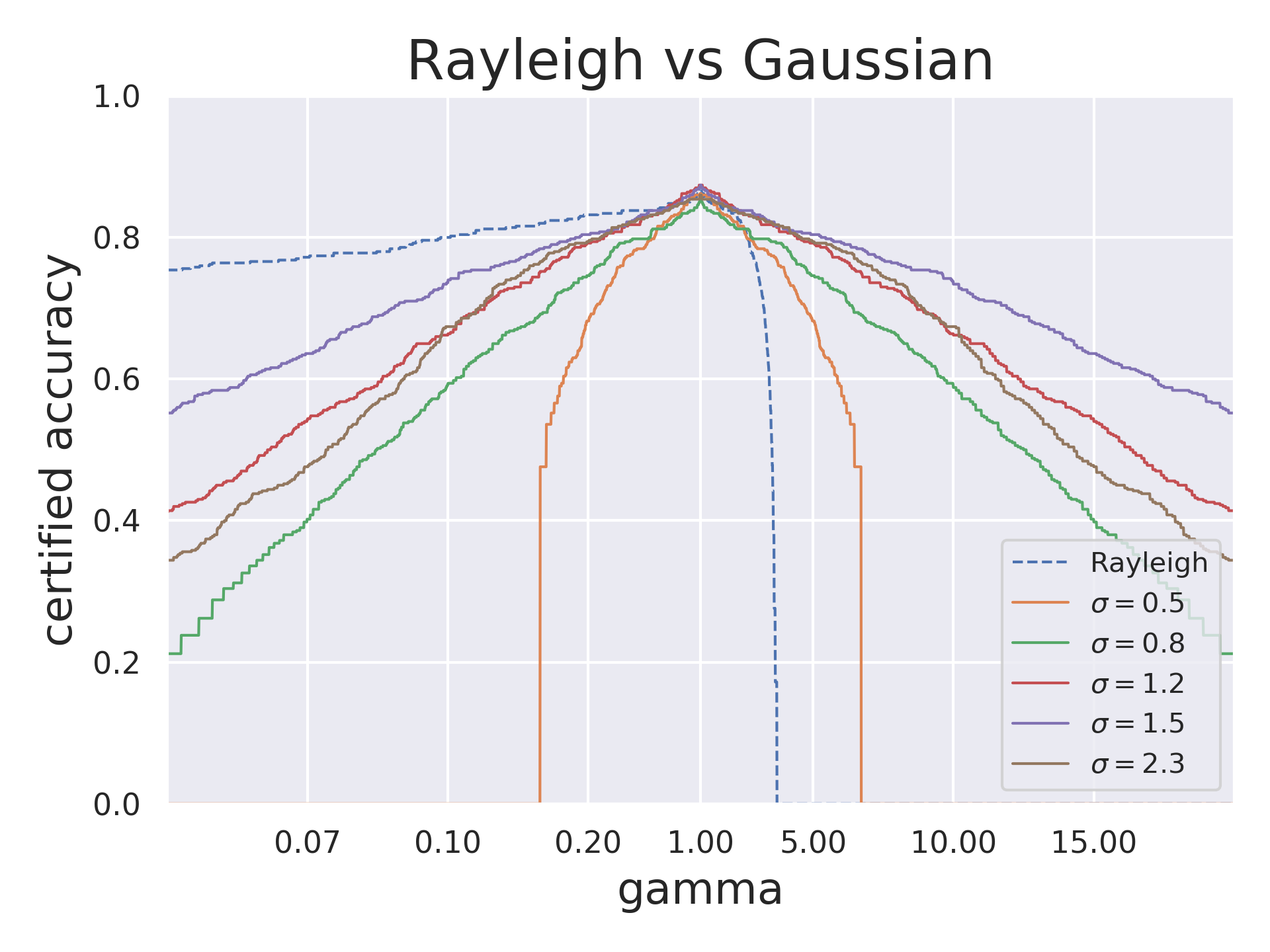}
\includegraphics[width=0.33\textwidth]{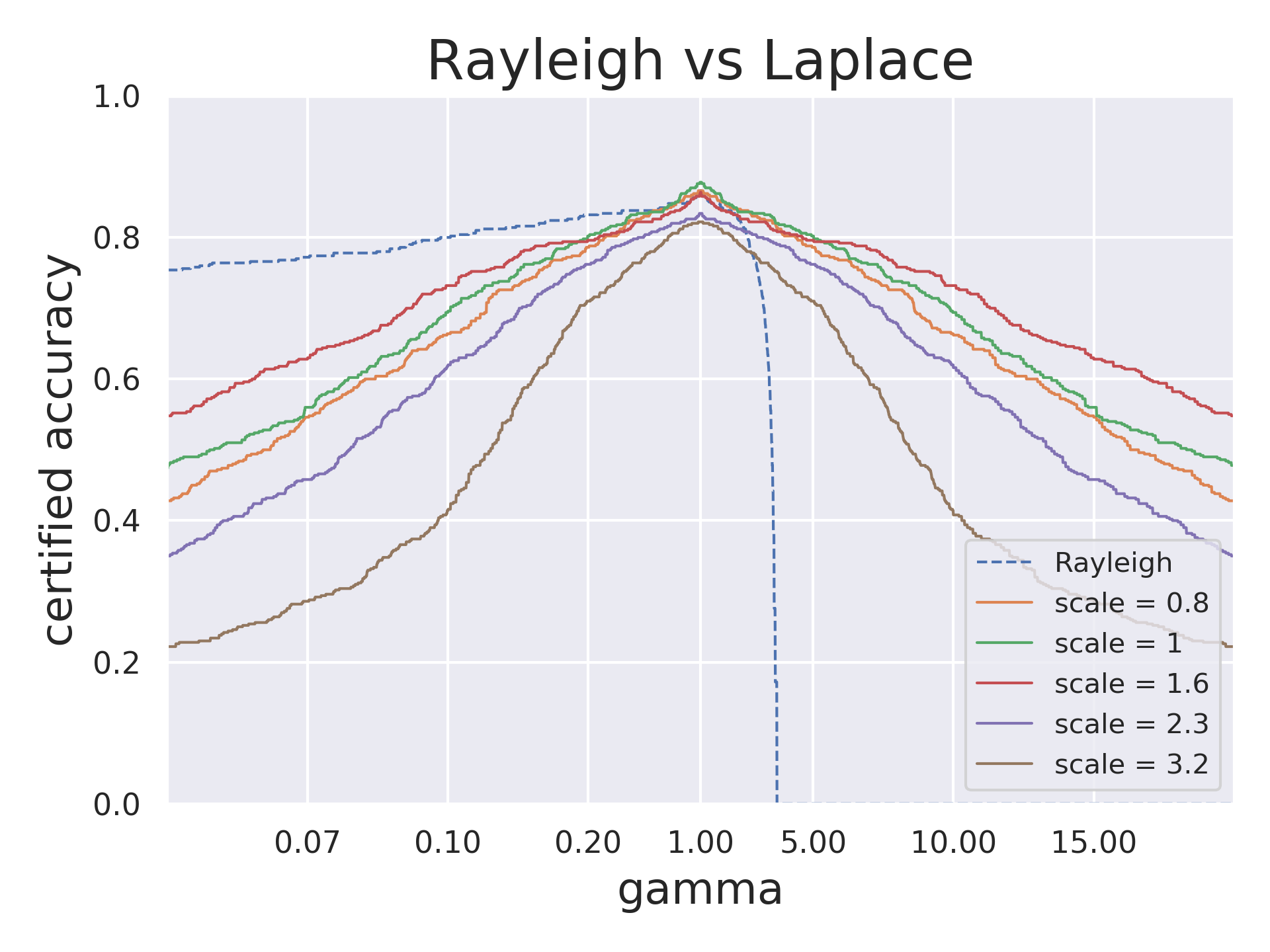}

\caption{Comparison of Rayleigh smoothing with uniform, Gaussian and Laplace smoothing for ResNet-110 on CIFAR-10 for the gamma correction perturbation. The base is equal to $e$ everywhere, the mean equals 0 for all alternate distributions and we only change the length of support (uniform), standard deviation (Gaussian) and scale (Laplace) parameters.}
\label{distributions}
\end{figure*}

In order to compare Rayleigh smoothing with other smoothing distributions we train several smoothed ResNet-110 classifiers on CIFAR-10 dataset with uniform, Gaussian and Laplace distributions following TSS protocol (including consistency regularization \cite{NEURIPS2020_77330e13}). We also use the TSS certification protocol with $\alpha = 0.001, n = 100000$ and 500 images randomly picked from the test set for certification. The results are presented at Fig.\ref{distributions}. To ease comparison the factor axes have been re-scaled for factors smaller than 1 to represent uniform multiplicative scale and a dotted line has been used for Rayleigh smoothing.

One can see that the Rayleigh smoothing outperforms all competitors for small gamma factors. Apparently asymmetric nature of the Rayleigh distribution allows to obtain much better certificates for small perturbation values than it would be possible with any conventional symmetric distribution. See Appendix \ref{E} for more detailed analysis of this phenomenon.

\subsection{Realistic setting}

Our study shows that if we convert gamma corrected images to some format with 8 bits per colour channel (RGB true colour), then the conversion error is usually too big to be ignored. Thus we need to make the base classifier $f$ robust in some $l_2$-ball with a radius larger then the conversion error for the most images. For these needs Fischer et al. \cite{fischer} propose to initially smooth the base classifier with the additive Gaussian noise to make it robust in $l_2$-norm. This new classifier can then be used as a base one for smoothing over transformation parameters. Thus instead of $n$ samples for the Monte-Carlo simulation algorithm we use $n_{\varepsilon}$ samples to simulate the smoothing in $l_2$-norm and $n_{\gamma}$ samples to simulate the smoothing over the gamma correction. Since we need $n_{\varepsilon}$ samples for every gamma correction sample, the resulting number of samples is equal to $n_{\varepsilon}\cdot n_{\gamma}$.

There are two types of guarantees that can be obtained by this method: 1) a distributional guarantee, when the conversion error is pre-calculated on the training dataset; 2) an individual guarantee, when we calculate the conversion error for each input at inference time.

For a given input image $x$ and the gamma correction transformation $G$ we define
\begin{equation}
\varepsilon(\beta,\gamma, x):= G_{\beta}\circ G_{\gamma}(x)-G_{\beta\cdot\gamma}(x)
\end{equation}
as the conversion error for gamma factors $\beta,\gamma$. In the idealized setting we assume $\varepsilon(\beta,\gamma,x)=0$ for all inputs. But now we need to estimate this value with some upper bound $E$.

In case of distributional guarantees the conversion error can be estimated with $E$ such that
\begin{equation}
q_E=\mathbb{P}_{x\sim D, \beta\sim Rayleigh}\left(\max\limits_{\gamma\in\Gamma}||\varepsilon(\beta,\gamma,x)||_2 \leq E\right),
\end{equation}
where $D$ is the data distribution, $\Gamma$ is a fixed interval of possible attacks and $1 - q_E$ is a desirable error rate.

The conversion errors $E$ are estimated for intervals $\Gamma$ of gamma factors $\Gamma_1 = [0.86, 1.15]$ and $\Gamma_2 = [0.71, 1.33]$ with the same error rate $q_E=0.9$ as $E_1 = 0.18$ and $E_2 = 0.22$ respectively. These intervals are chosen randomly but in such a way that they could be certificate intervals obtained with the Rayleigh smoothing for some input images, i.e. there exist such $\underline{p_A}, \overline{p_B}$ that the corresponding certificates obtained via Theorem \ref{main} equal to $\Gamma_1$ or $\Gamma_2$. Then we smooth the base classifier with the Gaussian noise with a deviation $\sigma=0.25$, that is found experimentally to deliver the best results for our base models and datasets.

The resulting mistake probability $\rho$ of that smoothed classifier on a ball $B_E(x)$ for a given $x$ can be estimated as the sum of mistake probabilities on all steps:
\begin{equation}
\rho \leq \alpha + 1 - q_E + \alpha_E,
\end{equation}
where $1-\alpha$ is the confidence of the base classifier, $q_E$ is the probabilistic guarantee for $E$ and $1-\alpha_E$ is the confidence with which $E$ is obtained. In all our experiments we set $\alpha_E = 0.01$.

\begin{figure}[t]
	\begin{center}
		\includegraphics[width = 0.45\textwidth]{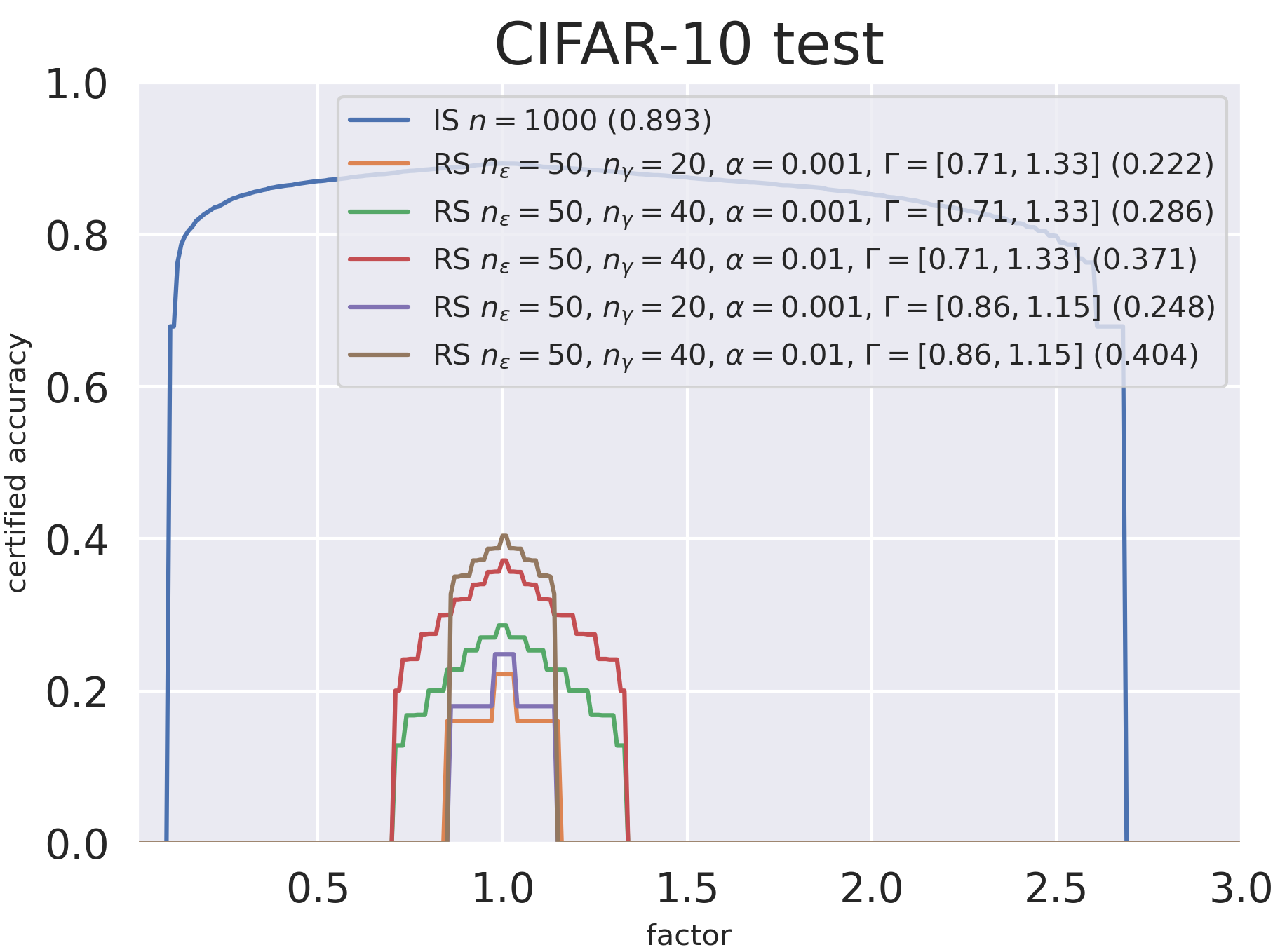}
	\end{center}
	\caption{Certified accuracy of smoothed ResNet-20 classifiers in the idealized (IS) and the realistic (RS) settings on CIFAR-10 for the gamma correction perturbation.}
	\label{real_cifar}
\end{figure}

This smoothed classifier is then smoothed with the Rayleigh distributed gamma correction. For the certification procedure the probabilities $\underline{p_A},\overline{p_B}$ are adjusted by $\rho$:
\begin{equation} \underline{p_A}^{\prime}=\underline{p_A} - \rho,\; \overline{p_B}^{\prime} = \overline{p_B} + \rho
\end{equation}
such that we take into consideration the mistake probability of the base classifier we have estimated previously. In order to preserve the correctness of conversion error estimations the obtained certificates are clipped to the selected attack intervals $\Gamma$.

We provide distributional guarantees for smoothed classifiers on CIFAR-10 for different numbers of samples $n_{\varepsilon},n_{\gamma}$, mistake probabilities $\alpha$ and attack intervals $\Gamma$ (Fig. \ref{real_cifar}). A significant drop in accuracy can be seen when we switch from the idealized setting to the realistic one. It happens because we are not able to preserve the same number of samples for smoothing over the gamma parameter (due to the computational complexity) as well as because our base classifier has to be robust in $l_2$-norm.

\begin{figure*}[t]

\subcaptionbox{$\gamma = 0.5$}{\includegraphics[width=0.33\textwidth]{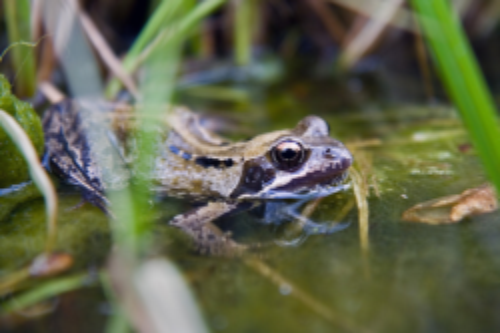}}
\subcaptionbox{$\gamma = 1$}{\includegraphics[width=0.33\textwidth]{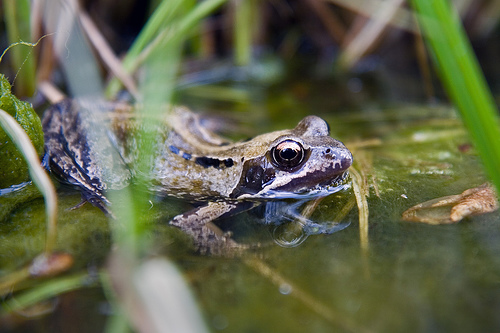}}
\subcaptionbox{$\gamma = 2$}{\includegraphics[width=0.33\textwidth]{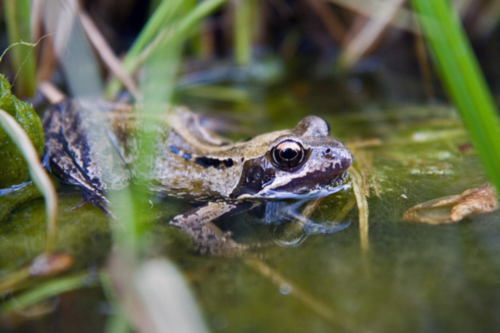}}

\caption{Demonstration of the scaling interpolation error on a frog image.}
\label{frog_scaling}
\end{figure*}

One can see that the smaller attack interval allows the lower estimation $E$ (which is expected) and thus the higher accuracy can be obtained. It should also be noticed that a greater mistake probability $\alpha$ increases the accuracy due to fewer ``abstain'' answers by both smoothed classifiers, but at the same time a greater $\alpha$ decreases the resulting confidence with which certificates are obtained.

The achieved certification results can be found in Appendix \ref{D}.

\subsection{Experiments with scaling interpolation error}

The other type of distortion we experiment with is a scaling interpolation error. By that we mean a transformation $S_{r}(x)$ which scales the image $x$ by factor $r$ followed by re-scaling to the original size. For a human eye it looks like a loss of clarity (Fig. \ref{frog_scaling}). Note that this transformation is not the same as scaling (e.g. like here \cite{li2020provable}) where there is no re-scaling to the original size and black padding is used instead. The resulting interpolation error is a perturbation we want our classifier to be robust against. It is clear that in the idealized setting scale factors multiply to each other when we scale an image several times and in that case the interpolation error does not exist. Thus this transformation is multiplicatively composable. But in reality we face non-zero interpolation errors that cannot be ignored. Though it is possible to handle this problem via double smoothing like it is done with the gamma correction, we find that the average interpolation error on CIFAR-10 images is too great to be handled via Gaussian smoothing with satisfactory accuracy. Fischer et al. \cite{fischer} encounter the same problem dealing with rotations. They propose several techniques to reduce interpolation errors (preprocessing with vignetting and low-pass filter, noise partitioning). Some of these techniques can be applied to our case but we find that the problem persists even for high-resolution images (e.g. ImageNet) that is not the case for rotations.

It should be noticed that down-scaling presents a greater challenge than up-scaling since the latter is theoretically invertible and results in smaller interpolation errors. Therefore it seems reasonable to use non-symmetric (with respect to multiplication) distributions for smoothing like the Rayleigh one, which has a distinct bias towards smaller values rather than bigger ones. One can expect such distributions to provide better certification results than the standard Gaussian and other symmetric distributions for perturbations whose symmetric parameter values do not cause the same level of distortion.

Of course one could use the TSS framework \cite{li2020provable} to deal with the interpolation errors but this approach does not require any smoothing distribution and thus it does not allow to exploit the asymmetrical nature of scaling distortions via asymmetrical probability distributions.

\section{Further research}

There are some recent works revealing the limitations of the standard Gaussian smoothing \cite{hayes,zhang}, especially in case of multidimensional perturbations with $l_p$-ball certificates for large $p$. It seems reasonable to search for alternative smoothing distributions not only among those suitable for additive parameters but also among distributions of multiplicative parameters. Indeed, it is interesting to find a good family of distributions for smoothing over multiplicative parameters, so that one could vary the variance like in additive case with uniform, Gaussian and Laplace distributions. Moreover, an example of the scaling perturbation shows that we may be more interested in asymmetric distributions which are able to exploit asymmetrical nature of certain transformations. There is a chance that it will allow to avoid the pitfalls and limitations of conventional smoothing distributions and achieve better results in the robustness certification tasks.

Also this technique can be applied to other types of perturbations and even to other types of data (audio signals, video, etc.). Therefore further experiments can include construction of certifiably robust audio and video classifiers based on the Rayleigh smoothing and comparison of them with those obtained via other methods.

\section{Related work}

In this section we want to discuss similar works in certification theory and point out some key differences between them and the present paper. Usually randomized smoothing is applied to classifiers having $R^d$ as their domain, resulting in the threat model when the signal is directly fed into the model without interpolation, rounding or clipping (like in \cite{li2020provable}). In real life it is unlikely that we will be able to perturb the original signal itself. Usually we have only a digital record of the signal (e.g. an image) and cannot access the original. Thus this setting seems unrealistic and we call it "idealized" in the paper. We go further and propose to consider "realistic" attacks that take into account limitations of digital image processing (limited channel width and bounds on minimal and maximal pixel values). Notice that gamma correction preserves pixel values within $[0,1]^3$ RGB-area  (unlike brightness or contrast changes) and thus can be easily certified against in "realistic" setting via above mentioned double smoothing technique.

The realistic setting results in a big certified accuracy drop, but we believe that it is a much more adequate way to simulate real-world attacks.

\section{Conclusion}

This work proposes a novel approach of randomized smoothing directly over multiplicative parameters of input transformations. We prove certification theorems and provide experimental comparison of the proposed method with conventional smoothing distributions for gamma correction transformations and show that it cannot be consistently outperformed by the latter. We also propose to consider realistic attacks which include additional conversion errors caused by the limited color channel width. As a result, the first certifiably robust image classifiers in the idealized and realistic settings are constructed for the gamma correction perturbation.

\bibliographystyle{named}
\bibliography{ijcai22}

\begin{thebibliography}{}

\bibitem[\protect\citeauthoryear{Agarwal and Mahajan}{2017}]{AGARWAL2017509}
Monika Agarwal and Rashima Mahajan.
\newblock Medical images contrast enhancement using quad weighted histogram
  equalization with adaptive gama correction and homomorphic filtering.
\newblock {\em Procedia Computer Science}, 115:509--517, 2017.
\newblock 7th International Conference on Advances in Computing \&
  Communications, ICACC-2017, 22-24 August 2017, Cochin, India.

\bibitem[\protect\citeauthoryear{Alfarra \bgroup \em et al.\egroup
  }{2021}]{alfarra2021deformrs}
Motasem Alfarra, Adel Bibi, Naeemullah Khan, Philip H.~S. Torr, and Bernard
  Ghanem.
\newblock Deformrs: Certifying input deformations with randomized smoothing,
  2021.

\bibitem[\protect\citeauthoryear{Athalye \bgroup \em et al.\egroup
  }{2018}]{pmlr-v80-athalye18a}
Anish Athalye, Nicholas Carlini, and David Wagner.
\newblock Obfuscated gradients give a false sense of security: Circumventing
  defenses to adversarial examples.
\newblock In Jennifer Dy and Andreas Krause, editors, {\em Proceedings of the
  35th International Conference on Machine Learning}, volume~80 of {\em
  Proceedings of Machine Learning Research}, pages 274--283, Stockholmsmässan,
  Stockholm Sweden, 10--15 Jul 2018. PMLR.

\bibitem[\protect\citeauthoryear{Carlini and
  Wagner}{2017}]{10.1145/3128572.3140444}
Nicholas Carlini and David Wagner.
\newblock Adversarial examples are not easily detected: Bypassing ten detection
  methods.
\newblock In {\em Proceedings of the 10th ACM Workshop on Artificial
  Intelligence and Security}, AISec '17, page 3–14, New York, NY, USA, 2017.
  Association for Computing Machinery.

\bibitem[\protect\citeauthoryear{Clopper and Pearson}{1934}]{clopper}
C.~J. Clopper and E.~S. Pearson.
\newblock The use of confidence or fiducial limits illustrated in the case of
  the binomial.
\newblock {\em Biometrika}, 26(4):404--413, 1934.

\bibitem[\protect\citeauthoryear{Cohen \bgroup \em et al.\egroup
  }{2019}]{cohen}
Jeremy~M. Cohen, Elan Rosenfeld, and J.~Zico Kolter.
\newblock Certified adversarial robustness via randomized smoothing.
\newblock {\em Proceedings of Machine Learning Research}, 97:1310--1320, 2019.

\bibitem[\protect\citeauthoryear{Deng \bgroup \em et al.\egroup
  }{2009}]{inproceedings}
Jia Deng, Wei Dong, Richard Socher, Li-Jia Li, Kai Li, and Fei-Fei Li.
\newblock Imagenet: a large-scale hierarchical image database.
\newblock {\em IEEE Conference on Computer Vision and Pattern Recognition},
  pages 248--255, 06 2009.

\bibitem[\protect\citeauthoryear{Fischer \bgroup \em et al.\egroup
  }{2020}]{fischer}
Marc Fischer, Maximilian Baader, and Martin Vechev.
\newblock Certified defense to image transformations via randomized smoothing.
\newblock {\em Advances in Neural Information Processing Systems},
  33:8404--8417, 2020.

\bibitem[\protect\citeauthoryear{Goodfellow \bgroup \em et al.\egroup
  }{2015}]{43405}
Ian Goodfellow, Jonathon Shlens, and Christian Szegedy.
\newblock Explaining and harnessing adversarial examples.
\newblock In {\em International Conference on Learning Representations}, 2015.

\bibitem[\protect\citeauthoryear{{Hayes}}{2020}]{hayes}
J.~{Hayes}.
\newblock Extensions and limitations of randomized smoothing for robustness
  guarantees.
\newblock In {\em 2020 IEEE/CVF Conference on Computer Vision and Pattern
  Recognition Workshops (CVPRW)}, pages 3413--3421, 2020.

\bibitem[\protect\citeauthoryear{Jeong and Shin}{2020}]{NEURIPS2020_77330e13}
Jongheon Jeong and Jinwoo Shin.
\newblock Consistency regularization for certified robustness of smoothed
  classifiers.
\newblock In H.~Larochelle, M.~Ranzato, R.~Hadsell, M.~F. Balcan, and H.~Lin,
  editors, {\em Advances in Neural Information Processing Systems}, volume~33,
  pages 10558--10570. Curran Associates, Inc., 2020.

\bibitem[\protect\citeauthoryear{Kurakin and Ian~Goodfellow}{2017}]{35184}
Alexey Kurakin and Samy~Bengio Ian~Goodfellow.
\newblock Adversarial machine learning at scale.
\newblock {\em ArXiv}, 2017.

\bibitem[\protect\citeauthoryear{Lecuyer \bgroup \em et al.\egroup
  }{2019}]{lecuyer}
M.~Lecuyer, V.~Atlidakis, R.~Geambasu, D.~Hsu, and S~Jana.
\newblock Certified robustness to adversarial examples with differential
  privacy.
\newblock {\em In IEEE Symposium on Security and Privacy (SP)}, 2019.

\bibitem[\protect\citeauthoryear{Li \bgroup \em et al.\egroup }{2019}]{li}
Bai Li, Changyou Chen, Wenlin Wang, and Lawrence Carin.
\newblock Certified adversarial robustness with additive noise.
\newblock {\em Advances in Neural Information Processing Systems}, 32, 2019.

\bibitem[\protect\citeauthoryear{Li \bgroup \em et al.\egroup
  }{2021}]{li2020provable}
Linyi Li, Maurice Weber, Xiaojun Xu, Luka Rimanic, Bhavya Kailkhura, Tao Xie,
  Ce~Zhang, and Bo~Li.
\newblock Tss: Transformation-specific smoothing for robustness certification.
\newblock In {\em Proceedings of the 2021 ACM SIGSAC Conference on Computer and
  Communications Security}, CCS '21, page 535–557, New York, NY, USA, 2021.
  Association for Computing Machinery.

\bibitem[\protect\citeauthoryear{Liu \bgroup \em et al.\egroup
  }{2018}]{Liu_2018_ECCV}
Xuanqing Liu, Minhao Cheng, Huan Zhang, and Cho-Jui Hsieh.
\newblock Towards robust neural networks via random self-ensemble.
\newblock In {\em Proceedings of the European Conference on Computer Vision
  (ECCV)}, September 2018.

\bibitem[\protect\citeauthoryear{Somasundaram and
  Kalavathi}{2011}]{Karuppanagounder}
Karuppanagounder Somasundaram and Palanisamy Kalavathi.
\newblock Medical image contrast enhancement based on gamma correction.
\newblock {\em International Journal of Knowledge Management and e-Learning},
  3:15--18, 02 2011.

\bibitem[\protect\citeauthoryear{Szegedy \bgroup \em et al.\egroup
  }{2014}]{szegedy2014intriguing}
Christian Szegedy, Wojciech Zaremba, Ilya Sutskever, Joan Bruna, Dumitru Erhan,
  J.~Ian Goodfellow, and Rob Fergus.
\newblock Intriguing properties of neural networks.
\newblock {\em international conference on learning representations}, 2014.

\bibitem[\protect\citeauthoryear{Veluchamy and
  Subramani}{2019}]{VELUCHAMY2019329}
Magudeeswaran Veluchamy and Bharath Subramani.
\newblock Image contrast and color enhancement using adaptive gamma correction
  and histogram equalization.
\newblock {\em Optik}, 183:329--337, 2019.

\bibitem[\protect\citeauthoryear{Wong and Kolter}{2018}]{wong}
Eric Wong and Zico Kolter.
\newblock Provable defenses against adversarial examples via the convex outer
  adversarial polytope.
\newblock In Jennifer Dy and Andreas Krause, editors, {\em Proceedings of the
  35th International Conference on Machine Learning}, volume~80 of {\em
  Proceedings of Machine Learning Research}, pages 5286--5295,
  Stockholmsmässan, Stockholm Sweden, 10--15 Jul 2018. PMLR.

\bibitem[\protect\citeauthoryear{Zhang \bgroup \em et al.\egroup
  }{2020}]{zhang}
Dinghuai Zhang, Mao Ye, Chengyue Gong, Zhanxing Zhu, and Qiang Liu.
\newblock Black-box certification with randomized smoothing: A functional
  optimization based framework.
\newblock {\em Advances in Neural Information Processing Systems},
  33:2316--2326, 2020.

\end{thebibliography}

\onecolumn
\appendix
\section*{Appendix}
\addcontentsline{toc}{section}{Appendix}
\renewcommand{\thesubsection}{\Alph{subsection}}
\subsection{Proofs}\label{A}

Here we prove the main theorems that we stated in the work.

\begin{theorem}
	Let ${x\in {\mathbb{R}^m}}$, ${f:\mathbb{R}^m \rightarrow Y}$ be a classifier, ${\psi_{\beta}:\mathbb{R}^m \rightarrow \mathbb{R}^m}$ be a multiplicatively composable  transformation for ${\beta\sim Rayleigh(\sigma)}$ and $g(x) = \arg\max\limits_{c}{\mathbb{P}_{\beta}(f\circ\psi_{\beta}(x)=c)}$. Denote $${p_A = \mathbb{P}_{\beta}(f\circ \psi_{\beta}(x)=c_A)},$$ $${p_B = \max\limits_{c_B \not= c_A}{ \mathbb{P}_{\beta}(f\circ\psi_{\beta}(x)}=c_B)}.$$ If
	$$p_A \geq \underline{p_A} > \overline{p_B}\geq p_B,$$
	then $g\circ\psi_{\gamma}(x)=c_A$ for all $\gamma$ satisfying $\gamma_1 < \gamma < \gamma_2$, where $\gamma_1, \gamma_2$ are the only solutions of the following equations
	$$F(\gamma_1^{-1}F^{-1}(\overline{p_B})) + F(\gamma_1^{-1}F^{-1}(1-\underline{p_A}))=1,$$
	$$F(\gamma_2^{-1}F^{-1}(\underline{p_A})) + F(\gamma_2^{-1}F^{-1}(1-\overline{p_B}))=1,$$
	$$F(z) = 1 - e^{-z^2/(2\sigma^2)} \text{ is the CDF of } \beta.$$
\end{theorem}

The proof is analogous to the one in \cite{fischer}. We try to preserve its structure and used abbreviations.

\begin{proof}
	Let us show that $$\mathbb{P}_{\beta}(f\circ \psi_{\beta\gamma}(x)=c_A)\stackrel{(1)}{\geq} \mathbb{P}_{\beta}(\beta\gamma\in A)$$
	$$\stackrel{(2)}{>} \mathbb{P}_{\beta}(\beta\gamma\in B)\stackrel{(3)}{\geq} \mathbb{P}_{\beta}(f\circ \psi_{\beta\gamma}(x)=c_B)$$
	for all $\gamma_1 < \gamma < \gamma_2$, where $$A = \{z\mid z\leq F^{-1}(\underline{p_A})\},\; B = \{z\mid z\geq F^{-1}(1-\overline{p_B})\}.$$
	Then the required statement for $\gamma \geq 1$ follows directly from these inequalities.
	
	1) Assume $p_{\zeta}(z)$ is the PDF of random variable $\zeta$. Then $$\mathbb{P}(f\circ\psi_{\beta\gamma}(x)=c_A)-\mathbb{P}(\beta\gamma\in A)$$
	$$= \int\limits_{\mathbb{R}}[f\circ\psi_z(x)=c_A]p_{\beta\gamma}(z)dz - \int\limits_{A}p_{\beta\gamma}(z)dz$$
	
	$$=\left(\int\limits_{\mathbb{R}\setminus A}[f\circ\psi_z(x)=c_A]p_{\beta\gamma}(z)dz + \int\limits_{A}[f\circ\psi_z(x)=c_A]p_{\beta\gamma}(z)dz\right)$$
	
	$$-\left(\int\limits_A[f\circ\psi_z(x)=c_A]p_{\beta\gamma}(z)dz + \int\limits_A[f\circ\psi_z(x)\not=c_A]p_{\beta\gamma}(z)dz\right)$$
	
	$$= \int\limits_{\mathbb{R}\setminus A}[f\circ\psi_z(x)=c_A]p_{\beta\gamma}(z)dz - \int\limits_A[f\circ\psi_z(x)\not=c_A]p_{\beta\gamma}(z)dz$$
	
	$$\stackrel{Lemma\, \ref{tech}}{\geq} t\left(\int\limits_{\mathbb{R}\setminus A}[f\circ\psi_z(x)=c_A]p_{\beta}(z)dz - \int\limits_A[f\circ\psi_z(x)\not=c_A]p_{\beta}(z)dz\right)$$
	
	$$=t\left(\int\limits_{\mathbb{R}}[f\circ\psi_z(x)=c_A]p_{\beta}(z)dz - \int\limits_Ap_{\beta}(z)dz\right)$$
	$$= t(p_A - \underline{p_A}) \geq 0.$$
	
	2) The inequality $$F\left(\gamma^{-1}F^{-1}(\underline{p_A})\right) = \mathbb{P}\left(\beta\gamma\leq F^{-1}(\underline{p_A})\right) = \mathbb{P}(\beta\gamma\in A) > \mathbb{P}(\beta\gamma\in B)$$
	
	$$= 1 - \mathbb{P}\left(\beta\gamma \leq F^{-1}(1-\overline{p_B})\right) = 1 - F\left(\gamma^{-1}F^{-1}(1-\overline{p_B})\right)$$
	clearly holds in case $\gamma = 1$. On the other hand there is such a huge $\gamma$ that it does not hold. From monotony and continuity of the function $F$ it follows that there is such a $\gamma_2$ that the inequality holds if and only if $1\leq\gamma < \gamma_2$. This $\gamma_2$ is the only solution of the following equation:
	$$F\left(\gamma^{-1}F^{-1}(\underline{p_A})\right) = 1 - F\left(\gamma^{-1}F^{-1}(1-\overline{p_B})\right).$$
	
	3) The proof is analogous to the proof for (1).
	
	The case $\gamma < 1$ is analogous, but with
	$$A = \{z\mid z \geq F^{-1}(1-\underline{p_A})\},\; B = \{z \mid z \leq F^{-1}(\overline{p_B})\}.$$
\end{proof}

\begin{lemma} \label{tech}
	There exists $t>0$ such that $p_{\beta\gamma}(z)\leq t \cdot p_{\beta}(z)$ for all $z\in A$. And further $p_{\beta\gamma}(z) > t\cdot p_{\beta}(z)$ for all $a\in \mathbb{R}\setminus A$.
\end{lemma}

\begin{proof}
	$$\frac{p_{\beta\gamma}(z)}{p_{\beta}(z)} = \gamma^{-2}\cdot exp\left(-\frac{z^2}{2\sigma^2\gamma^2} + \frac{z^2}{2\sigma^2}\right) = \gamma^{-2} \cdot exp\left(\frac{z^2(\gamma^2 - 1)}{2\sigma^2\gamma^2}\right) \leq t$$
	
	$$\Leftrightarrow \frac{z^2(\gamma^2 - 1)}{2\sigma^2\gamma^2} \leq ln(t\gamma^2)$$ $$\Leftrightarrow (\gamma^2 - 1)z^2 \leq 2\sigma^2\gamma^2 ln(t\gamma^2)$$ $$\Leftrightarrow z \leq \frac{\sigma\gamma\sqrt{2(\gamma^2 - 1)ln(t\gamma^2)}}{(\gamma^2 - 1)}.$$ 
	
	What is the lowest $t$ such that $\frac{p_{\beta\gamma}(z)}{p_{\beta}(z)} \leq t$ for all $z\in A$?\\
	Because of $A = \{z \mid z \leq F^{-1}(\underline{p_A})\}$ we know that $z \leq F^{-1}(\underline{p_A})$. Does there exist such a threshold $t$ that both upper bounds coincide? Yes, namely
	
	$$t = \gamma^{-2} exp\left(\frac{(\gamma^2-1)\left(F^{-1}(\underline{p_A})\right)^2}{2\sigma^2\gamma^2}\right).$$
	
\end{proof}

The theorem can be generalized to the case of transformations with multiple multiplicative parameters.

\begin{theorem}
	Let $x\in {\mathbb{R}^m}$, $f:\mathbb{R}^m \rightarrow Y$ be a classifier, ${\psi_{\beta}:\mathbb{R}^m \rightarrow \mathbb{R}^m},\\ {\beta=(\beta_1, ..., \beta_n)^{T}}$ be a multiplicatively composable  transformation for independent random variables $\beta_i\sim Rayleigh(\sigma)$ and ${g(x) = \arg\max\limits_{c}{\mathbb{P}_{\beta}(f\circ\psi_{\beta}(x)=c)}}$
	Denote $${p_A = \mathbb{P}_{\beta}(f\circ \psi_{\beta}(x)=c_A)},$$ $${p_B = \max\limits_{c_B \not= c_A}{ \mathbb{P}_{\beta}(f\circ\psi_{\beta}(x)}=c_B)}.$$ If
	$$p_A \geq \underline{p_A} > \overline{p_B}\geq p_B,$$
	then $g\circ\psi_{\gamma}(x)=c_A$ for all $\gamma\in\Omega$, where $\Omega$ is a region defined by the following inequality
	
	$$\mathbb{P}((\gamma_1^{2}-1)\beta_1^2 +...+ (\gamma_n^{2}-1)\beta_n^2 \leq r) >
	\mathbb{P}((\gamma_1^{2}-1)\beta_1^2 +...+ (\gamma_n^{2}-1)\beta_n^2 \geq \theta),$$
	where $r$ and $\theta$ are the only solutions of the following equations
	$$\mathbb{P}((1-\gamma_1^{-2})\beta_1^2 +...+ (1-\gamma_n^{-2})\beta_n^2 \leq r)=\underline{p_A},$$
	$$\mathbb{P}((1-\gamma_1^{-2})\beta_1^2 +...+ (1-\gamma_n^{-2})\beta_n^2 \geq \theta)=\overline{p_B}.$$
\end{theorem}

\begin{proof}
	Let us show that $$\mathbb{P}_{\beta}(f\circ \psi_{\beta\gamma}(x)=c_A)\stackrel{(1)}{\geq} \mathbb{P}_{\beta}(\beta\gamma\in A)\stackrel{(2)}{>} \mathbb{P}_{\beta}(\beta\gamma\in B)\stackrel{(3)}{\geq} \mathbb{P}_{\beta}(f\circ \psi_{\beta\gamma}(x)=c_B)$$
	for all $\gamma\in\Omega$, where $$A =\{(z_1,...,z_n)\mid (1-\gamma_1^{-2})z_1^2 +...+ (1-\gamma_n^{-2})z_n^2\leq r\},$$
	$$B =\{(z_1,...,z_n)\mid (1-\gamma_1^{-2})z_1^2 +...+ (1-\gamma_n^{-2})z_n^2\geq \theta\}.$$
	Then the required statement follows directly from these inequalities.
	
	1) Assume $p_{\zeta}(z)$ is the PDF of a random variable $\zeta$. Then $$\mathbb{P}(f\circ\psi_{\beta\gamma}(x)=c_A)-\mathbb{P}(\beta\gamma\in A) = \int\limits_{\mathbb{R}^n}[f\circ\psi_z(x)=c_A]p_{\beta\gamma}(z)dz - \int\limits_{A}p_{\beta\gamma}(z)dz$$
	
	$$=\left(\int\limits_{\mathbb{R}^n\setminus A}[f\circ\psi_z(x)=c_A]p_{\beta\gamma}(z)dz + \int\limits_{A}[f\circ\psi_z(x)=c_A]p_{\beta\gamma}(z)dz\right)$$
	
	$$-\left(\int\limits_A[f\circ\psi_z(x)=c_A]p_{\beta\gamma}(z)dz + \int\limits_A[f\circ\psi_z(x)\not=c_A]p_{\beta\gamma}(z)dz\right)$$
	
	$$= \int\limits_{\mathbb{R}^n\setminus A}[f\circ\psi_z(x)=c_A]p_{\beta\gamma}(z)dz - \int\limits_A[f\circ\psi_z(x)\not=c_A]p_{\beta\gamma}(z)dz$$
	
	$$\stackrel{Lemma\, \ref{tech_mult}}{\geq} t\left(\int\limits_{\mathbb{R}^n\setminus A}[f\circ\psi_z(x)=c_A]p_{\beta}(z)dz - \int\limits_A[f\circ\psi_z(x)\not=c_A]p_{\beta}(z)dz\right)$$
	
	$$=t\left(\int\limits_{\mathbb{R}^n}[f\circ\psi_z(x)=c_A]p_{\beta}(z)dz - \int\limits_Ap_{\beta}(z)dz\right) = t(p_A - \underline{p_A}) \geq 0.$$
	
	2) The inequality $$\mathbb{P}\left((\gamma_1^{2}-1)\beta_1^2+...+(\gamma_n^{2}-1)\beta_n^2 \leq r\right) = \mathbb{P}(\beta\gamma\in A)$$
	$$> \mathbb{P}(\beta\gamma\in B)= \mathbb{P}\left((\gamma_1^{2}-1)\beta_1^2+...+(\gamma_n^{2}-1)\beta_n^2 \geq \theta\right)$$
	holds for all $\gamma\in\Omega$ by the definitions of $A,B$ and $\Omega$.
	
	3) The proof is analogous to the proof for (1).
\end{proof}

\begin{lemma} \label{tech_mult}
	There exists $t>0$ such that $p_{\beta\gamma}(z)\leq t \cdot p_{\beta}(z)$ for all $z\in A$. And further $p_{\beta\gamma}(z) > t\cdot p_{\beta}(z)$ for all $a\in \mathbb{R}^n\setminus A$.
\end{lemma}

\begin{proof}
	$$\frac{p_{\beta\gamma}(z)}{p_{\beta}(z)} =
	\frac{p_{\beta_1\gamma_1}(z_1)\cdot...\cdot p_{\beta_n\gamma_n}(z_n)}{p_{\beta_1}(z_1)\cdot ...\cdot p_{\beta_n}(z_n)}$$
	$$=(\gamma_1...\gamma_n)^{-2}\cdot exp\left(-\frac{z_1^2}{2\sigma^2\gamma_1^2} + \frac{z_1^2}{2\sigma^2}\right)\cdot ...\cdot exp\left(-\frac{z_n^2}{2\sigma^2\gamma_n^2} + \frac{z_n^2}{2\sigma^2}\right)$$
	$$= (\gamma_1...\gamma_n)^{-2} \cdot exp\left(\frac{z_1^2(1-\gamma_1^{-2})}{2\sigma^2}+...+\frac{z_n^2(1-\gamma_n^{-2})}{2\sigma^2}\right) \leq t$$
	
	$$\Leftrightarrow z_1^2(1-\gamma_1^{-2})+...+z_n^2(1-\gamma_n^{-2}) \leq 2\sigma^2 ln(t(\gamma_1...\gamma_n)^2).$$
	
	What is the lowest $t$ such that $\frac{p_{\beta\gamma}(z)}{p_{\beta}(z)} \leq t$ for all $z\in A$?\\
	Because of $A =\{(z_1,...,z_n)\mid (1-\gamma_1^{-2})z_1^2 +...+ (1-\gamma_n^{-2})z_n^2\leq r\}$ we know that $(1-\gamma_1^{-2})z_1^2 +...+ (1-\gamma_n^{-2})z_n^2\leq r$. Does there exist such a threshold $t$ that both upper bounds coincide? Yes, namely
	
	$$t = (\gamma_1...\gamma_n)^{-2} exp\left(\frac{r}{2\sigma^2}\right).$$
\end{proof}

We have found the region $\Omega$ in which the classifier is robust. But how to compute this region in practice? Note, that the $\gamma=(1,...,1)$ point does not necessary lie in $\Omega$. But of course we can add it to $\Omega$ as the classifier has the same outputs on it. Then it can be shown that this region is bounded.
Thus the region defining inequality can be solved numerically.

\subsection{Alternative proofs}\label{B}

The above theorems can be proven via a slightly modified functional optimization based framework \cite{zhang}. Let $f^{\#}:\mathbb{R}^m\rightarrow \{0,1\}$ be a binary classifier, $\mathcal{F}$ be a function class that is known to include $f^{\#}$, $\psi_{z}:\mathbb{R}^m\rightarrow\mathbb{R}^m$ be a multiplicative composable transformation with a scalar parameter $z$ and $f_{\pi_0}^{\#}(x_0):=\mathbb{E}_{z\sim\pi_0}[f^{\#}(\psi_z(x_0))]$ be the smoothed classifier over $\psi_z$ with the probability distribution $\pi_0$. Assume $f^{\#}_{\pi_0}(x_0)>\frac{1}{2}$. Then we want to certify that $f^{\#}_{\pi_0}(\psi_{\delta}(x_0))>\frac{1}{2}$ for every $\delta\in \mathcal{B}=[\gamma_1, \gamma_2]$. The constrained adversarial certification framework yields the following lower bound:
$$\min\limits_{\delta\in \mathcal{B}}f^{\#}_{\pi_0}(\psi_{\delta}(x_0)) \geq \min\limits_{f\in \mathcal{F}}\min\limits_{\delta\in \mathcal{B}}\left\{f_{\pi_0}(\psi_{\delta}(x_0))\; s.t.\; f_{\pi_0}(x_0)=f^{\#}_{\pi_0}(x_0)\right\}$$
$$\geq \min\limits_{f\in \mathcal{F}}\min\limits_{\delta\in \mathcal{B}}\max\limits_{\lambda\in\mathbb{R}}\left\{f_{\pi_0}(\psi_{\delta}(x_0)) - \lambda(f_{\pi_0}(x_0)-f^{\#}_{\pi_0}(x_0))\right\}$$
$$\geq \max\limits_{\lambda\in\mathbb{R}}\min\limits_{f\in \mathcal{F}}\min\limits_{\delta\in \mathcal{B}}\left\{f_{\pi_0}(\psi_{\delta}(x_0)) - \lambda(f_{\pi_0}(x_0)-f^{\#}_{\pi_0}(x_0))\right\}$$
$$= \max\limits_{\lambda\in\mathbb{R}}\left\{\lambda f^{\#}_{\pi_0}(x_0)-\max\limits_{\delta\in \mathcal{B}}\mathbb{D}_F (\lambda\pi_0||\pi_{\delta})\right\},$$
where $\pi_{\delta}$ is the distribution of $\psi_{\delta}(z)$ when $z\sim\pi_0$ and we define the discrepancy term
$$\mathbb{D}_{\mathcal{F}} (\lambda\pi_0||\pi_{\delta}):= \max\limits_{f\in \mathcal{F}}\left\{\lambda f_{\pi_0}(x_0) - f_{\pi_{\delta}}(x_0)\right\}.$$

If $\mathcal{F} = \{f: f(x)\in[0,1],\; x\in\mathbb{R}^m\}$, then
$$\mathbb{D}_{\mathcal{F}} (\lambda\pi_0||\pi_{\delta}) = \int(\lambda\pi_0(z) - \pi_{\delta}(z))_+dz,$$
where $(t)_+=\max\{0,t\}$.

With strong duality we have
$$\max\limits_{\lambda\in\mathbb{R}}\left\{\lambda f^{\#}_{\pi_0}(x_0)-\max\limits_{\delta\in \mathcal{B}}\mathbb{D}_{\mathcal{F}} (\lambda\pi_0||\pi_{\delta})\right\} \geq \min\limits_{\delta\in \mathcal{B}}\max\limits_{\lambda \geq 0}\left\{\lambda f^{\#}_{\pi_0}(x_0)-\int(\lambda\pi_0(z) - \pi_{\delta}(z))_+dz\right\}$$

So, all we need to do is to find such an interval $\mathcal{B}=[\gamma_1, \gamma_2]$, that
$$\min\limits_{\delta\in \mathcal{B}}\max\limits_{\lambda\geq 0}\left\{\lambda f^{\#}_{\pi_0}(x_0)-\int(\lambda\pi_0(z) - \pi_{\delta}(z))_+dz\right\} > \frac{1}{2}.$$

In our case $f^{\#}_{\pi_0}(x_0)$ is the probability $p_A$ of top 1 class $A$ which we estimate with $\underline{p_A}$, $\pi_0$ has the Rayleigh distribution and $$\pi_0(z) = \frac{z}{\sigma^2}e^{\frac{-z^2}{2\sigma^2}},\; \pi_{\delta}(z) = \frac{z}{\sigma^2\delta^2}e^{\frac{-z^2}{2\sigma^2\delta^2}}.$$

Let $$C_{\lambda}:=\{z: \lambda\pi_0(z) \geq \pi_{\delta}(z)\}=\{z:\lambda\delta^2 e^{\frac{-z^2}{2\sigma^2}} \geq e^{\frac{-z^2}{2\sigma^2\delta^2}}\}$$
$$= \{z:\frac{-z^2}{2\sigma^2} + \ln{\lambda\delta^2} \geq \frac{-z^2}{2\sigma^2\delta^2}\} = \{z: z^2(1-\delta^2) \geq -2\sigma^2\delta^2\ln{\lambda\delta^2}\}.$$

First, we find the lower bound $\gamma_1$. For this purpose we consider $\mathcal{B}=[\gamma_1, 1]$ which implies $\delta \leq 1$. In this case $$C_{\lambda} = \{z: z^2 \geq \frac{-2\sigma^2\delta^2\ln{\lambda\delta^2}}{1-\delta^2}\}.$$
$$\max\limits_{\lambda\geq 0}\left\{\lambda f^{\#}_{\pi_0}(x_0)-\int(\lambda\pi_0(z) - \pi_{\delta}(z))_+dz\right\} \geq \max\limits_{\lambda\geq 0}\left\{\lambda\underline{p_A} - \int_{C_{\lambda}}(\lambda\pi_0(z) - \pi_{\delta}(z))dz\right\}$$
$$=\max\limits_{\lambda\geq 0}\left\{\lambda\underline{p_A} - \lambda\mathbb{P}\left(\pi_0^2 \geq \frac{-2\sigma^2\delta^2\ln{\lambda\delta^2}}{1-\delta^2}\right) + \mathbb{P}\left(\pi_0^2 \geq \frac{-2\sigma^2\ln{\lambda\delta^2}}{1-\delta^2}\right)\right\}$$
$$=\max\limits_{0 \leq \lambda \leq \delta^{-2}}\left\{\lambda\underline{p_A} - \lambda\mathbb{P}\left(\pi_0^2 \geq \frac{-2\sigma^2\delta^2\ln{\lambda\delta^2}}{1-\delta^2}\right) + \mathbb{P}\left(\pi_0^2 \geq \frac{-2\sigma^2\ln{\lambda\delta^2}}{1-\delta^2}\right)\right\}$$
$$= \max\limits_{0 \leq \lambda \leq \delta^{-2}}\left\{\lambda\underline{p_A} - \lambda\left(1 - F\left(\sqrt{\frac{-2\sigma^2\delta^2\ln{\lambda\delta^2}}{1-\delta^2}}\right)\right) + 1 - F\left(\sqrt{\frac{-2\sigma^2\ln{\lambda\delta^2}}{1-\delta^2}}\right)\right\},$$
where $F(\cdot)$ is the CDF of the Rayleigh distribution with the scale parameter $\sigma$.

Let

$$G(\delta,\lambda) := \lambda\underline{p_A} - \int(\lambda\pi_0(z) - \pi_{\delta}(z))_+dz.$$

For any $\delta\in \mathcal{B}$ $G$ is concave with respect to $\lambda$. So for any $\delta$ $G$ has the only maximum located at $\lambda_{\delta}$ which is the solution of $\frac{\partial G(\delta,\lambda)}{\partial \lambda}=0$. Simple calculations show that

$$\lambda_{\delta} = \frac{1}{\delta^2}\underline{p_A}^{\frac{1-\delta^2}{\delta^2}}.$$

Thus

$$\min\limits_{\delta\in \mathcal{B}}\max\limits_{\lambda \geq 0}G(\delta, \lambda) = \min\limits_{\delta\in \mathcal{B}}G(\delta, \lambda_{\delta})$$
$$= \min\limits_{\delta\in \mathcal{B}}\left\{\lambda_{\delta}\underline{p_A} - \lambda_{\delta}\left(1 - F\left(\sqrt{\frac{-2\sigma^2\delta^2\ln{\lambda_{\delta}\delta^2}}{1-\delta^2}}\right)\right) + 1 - F\left(\sqrt{\frac{-2\sigma^2\ln{\lambda_{\delta}\delta^2}}{1-\delta^2}}\right) \right\}$$
$$=\min\limits_{\delta\in [\gamma_1, 1]}\underline{p_A}^{\delta^{-2}} = \underline{p_A}^{\gamma_1^{-2}} > \frac{1}{2} \Leftrightarrow \gamma_1 > \left(\log{\underline{p_A}}{\frac{1}{2}}\right)^{-\frac{1}{2}} = \frac{F^{-1}(1-\underline{p_A})}{F^{-1}(\frac{1}{2})}$$
which coincides with the result obtained by the previous method when we use a trivial bound on probability of top 2 class $\overline{p_B} = 1 - \underline{p_A}$.

The upper bound $\gamma_2$ can be calculated in the same way.

The result can be generalized to multi-class case as it was shown in \cite{zhang}.

There is a third way to prove the main theorem using TSS framework \cite{li2020provable}. We do not present it here but notice that it results in the same certificates as the two above do.

\subsection{Experiments with a scale parameter}
\label{C}

It is natural to assume that the best smoothed classifiers can be obtained with probability distributions centered at the point of identity transformation. But, unlike conventional distributions, the Rayleigh distribution has different median and mean values. Thus we experiment with the scale parameter $\sigma$ and choose it either to make the median equal to 1 ($\sigma = 1/\sqrt{2\log2}$) or to make the mean equal to 1 ($\sigma = \sqrt{2/\pi}$). Our evaluations for ResNet-110 on CIFAR-10 show that centering at the unit median is beneficial for resulting certificates (Fig.\ref{gamma_rayleigh}). Thus in all our experiments we set the scale parameter to $1/\sqrt{2\log2}$.

\begin{figure}[h]
	\begin{center}
		\includegraphics[width = 0.6\textwidth]{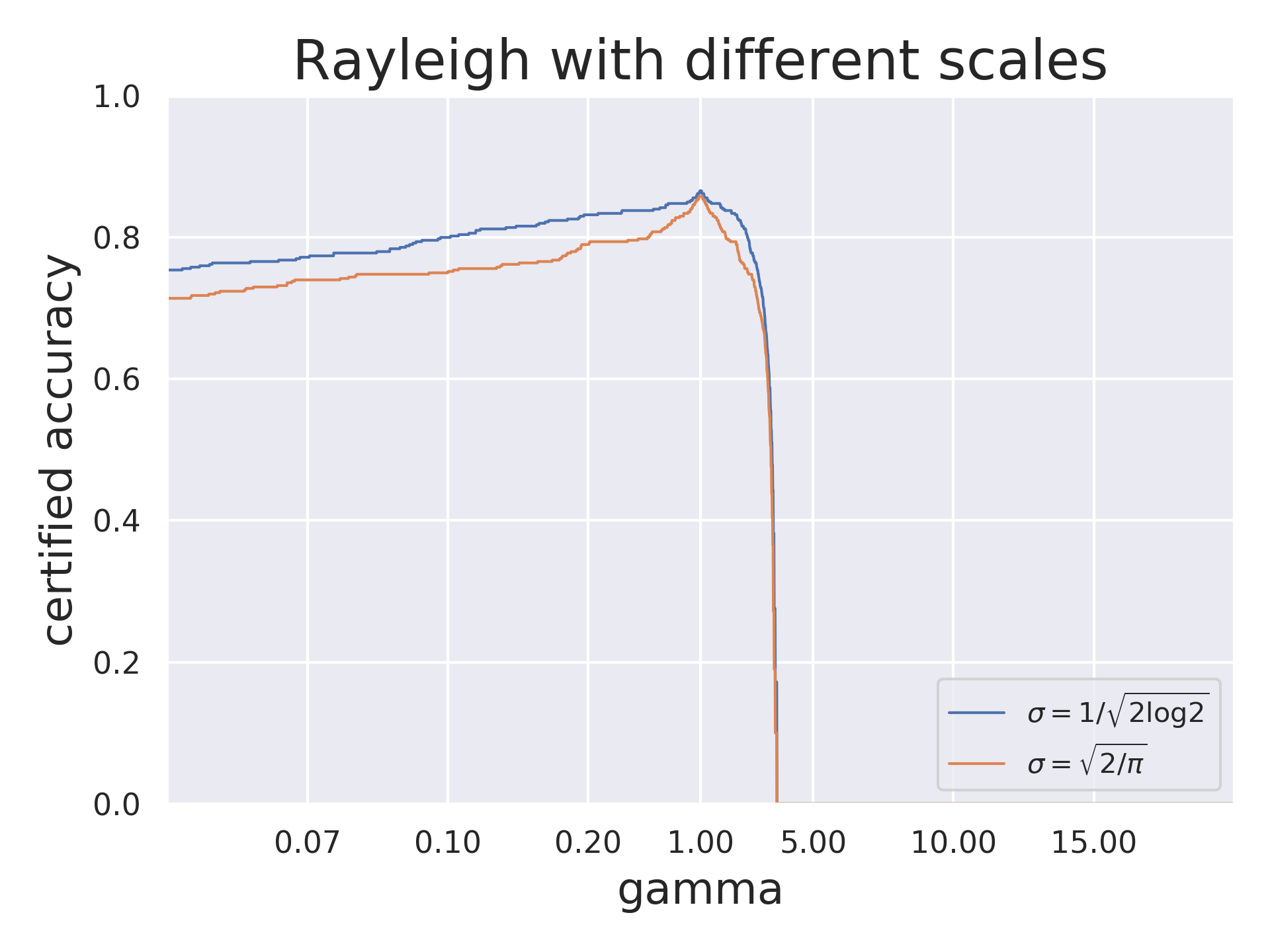}
	\end{center}
	\caption{Certification results for two smoothed ResNet-110s with different scales of Rayleigh distribution on CIFAR-10 dataset}
	\label{gamma_rayleigh}
\end{figure}

\subsection{Certification results}\label{D}

In this section we present obtained certificates against gamma-correction attacks for some attack intervals on CIFAR-10 and ImageNet datasets for smoothed ResNet-20 and ResNet-50 respectively in a table form (Table \ref{table}). The attack intervals $(\gamma_1, \gamma_2)$ are chosen randomly but in such a way that they could be certificates obtained via Rayleigh smoothing.

\begin{table*}[h]
\centering
\begin{tabular}{@{}llrrrrrr@{}}
    \toprule
	Dataset & Setting & $\gamma_1$ & $\gamma_2$ & $n_{\varepsilon}$ & $n_{\gamma}$ & $\alpha$ & Acc\\ \midrule
        CIFAR-10   & Idealized & 0.154 & 2.44 &   - & 1000 & 0.001 & 0.81\\
		 &   Realistic  & 0.710 & 1.33 &50 & 40 & 0.001 & 0.13\\
		 & & 0.710 & 1.33 & 50 & 40 & 0.010 & 0.20\\
		    & & 0.860 & 1.15 & 50 & 20 & 0.001 & 0.18\\
		 &     & 0.860 & 1.15 & 50 & 40 & 0.010 & 0.33\\
		ImageNet& Idealized  & 0.100 & 2.68 & - & 1000 & 0.001 & 0.63\\ \bottomrule
\end{tabular}
\caption{The certification results for the smoothed classifiers. Here ``Acc'' means the portion of images that were correctly classified and whose certificates contain an attack interval $\Gamma = (\gamma_1, \gamma_2)$.}
\label{table}
\end{table*}

\subsection{Asymmetrical distributions for smoothing}\label{E}

The present work demonstrates that randomized smoothing with asymmetrical distributions may be beneficial for certain values of transformation parameters. For example a Rayleigh distributed random variable with unit median has a distinct bias towards small values (Fig. \ref{rayleigh_pdf}). Even though we center the distribution such that the median equals 1, values greater that 1 concentrate in an interval $[1,4]$ while small values often appear near 0. This causes asymmetric certificates favoring smaller values.

It seems logical to expect that smoothing with the inverse Rayleigh distribution ($1/Rayleigh(\sigma)$) will produce an inverse effect when greater values have better certificates. Our experiments show that indeed this is the case (Fig. \ref{cifar10_gamma_inverse_rayleigh}). Yet one can notice that obtained via inverse Rayleigh smoothing certificates for great values are worse than those obtained via Rayleigh smoothing for corresponding small values. The reason behind it is asymmetrical nature of gamma correction. Consider the mean pixel value $x = 0.5$. Then two gamma correction values $\gamma > 1$ and $1/\gamma < 1$ cause different amount of distortion. For example $\gamma = 10$ results in $x^{\gamma} = 0.5^{10} \approx 0.00098$ and $x^{1/\gamma} = 0.5^{0.1} \approx 0.933$. That is for great $\gamma$ the result $x^{\gamma}$ is closer to 0 than $x^{1/\gamma}$ to 1. This leads to the fact that gamma correction with a great factor transfers different images closer to each other in $l_2$ space than a transformation with the corresponding small value does. Obviously closer images are harder to distinguish and a base classifier makes more mistakes causing worse certificates.

This research shows that asymmetrical distributions may be preferable in case of asymmetric perturbations (e.g. gamma-correction, scaling) where transformations with parameters of similar magnitude may cause different levels of distortion.

\begin{figure}[h]
	\begin{center}
		\includegraphics[width = 0.5\textwidth]{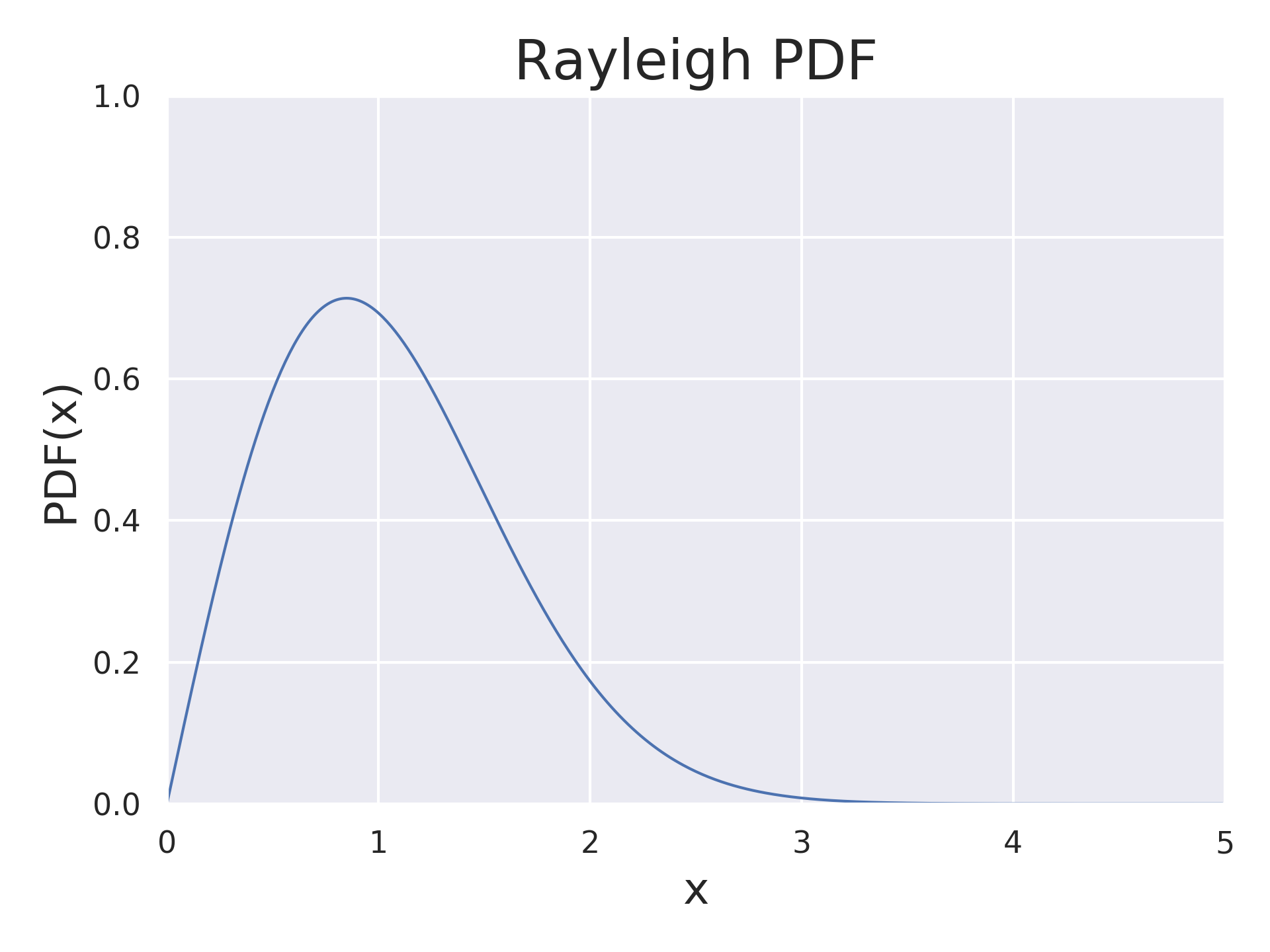}
	\end{center}
	\caption{Probability density function (PDF) for Rayleigh distributed random value with unit median}
	\label{rayleigh_pdf}
\end{figure}

\begin{figure}[t]
	\begin{center}
		\includegraphics[width = 0.6\textwidth]{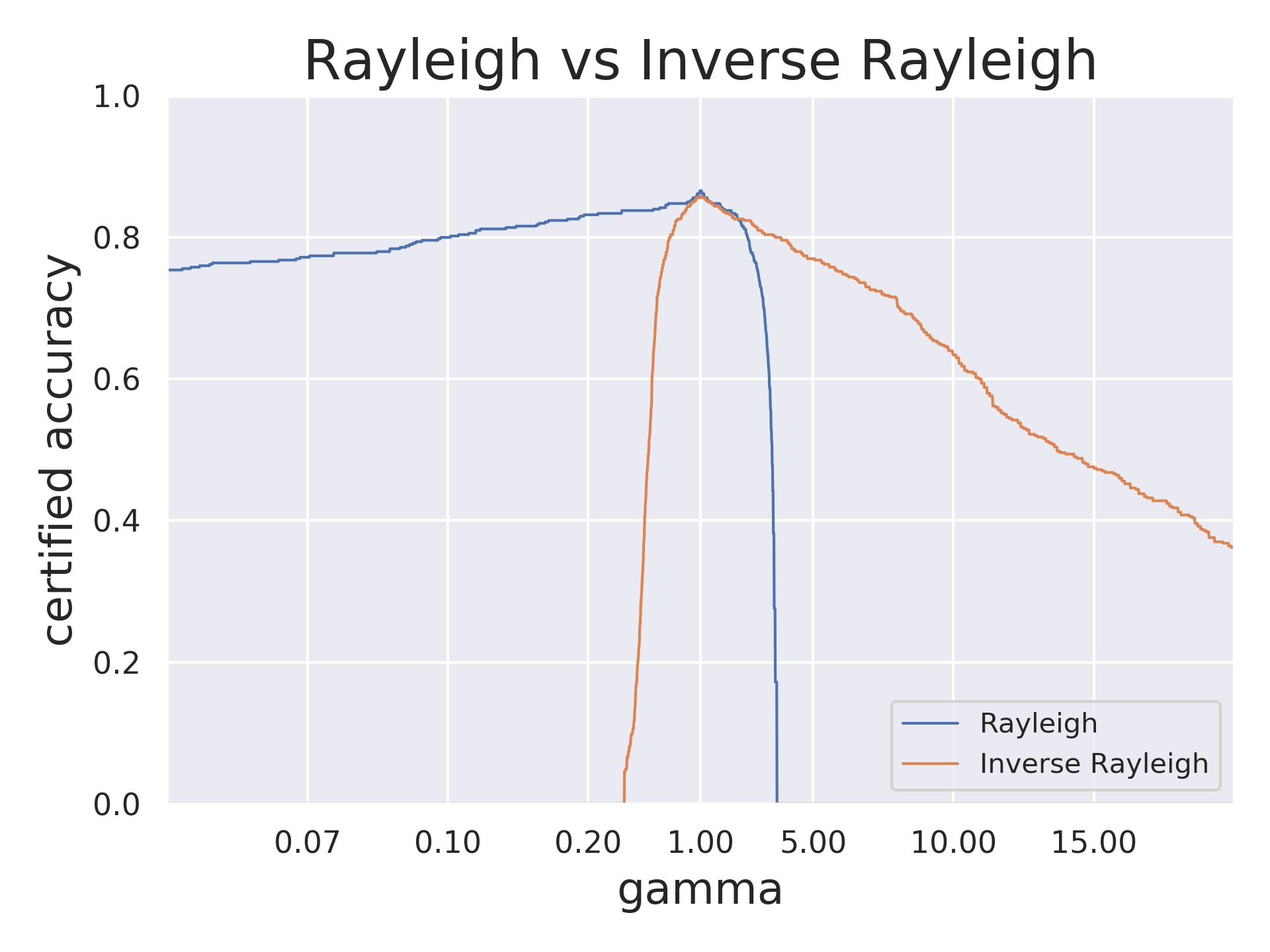}
	\end{center}
	\caption{Certified accuracy for ResNet-110 smoothed with Rayleigh and inverse Rayleigh distributions}
	\label{cifar10_gamma_inverse_rayleigh}
\end{figure}

\end{document}